\documentclass[twoside,11pt]{article}

\usepackage[preprint]{jmlr2e}

\usepackage{graphicx}
\usepackage{bm}
\usepackage{amsmath}
\usepackage{amscd}
\usepackage{stmaryrd}
\usepackage{mathtools}
\usepackage{appendix}
\usepackage{bm}

\newcommand{\av}[2][\,]{\mathbb{E}_{#1}\!\left[ {\strut #2} \right]}

\newcommand{\e}[1]{\mathrm{e}^{{\, #1}}}
\newcommand{\dd}{\mathrm{d}}
\newcommand{\ii}{\mathrm{i}}
\newcommand{\Prob}[1]{\mathbb{P}\left( \strut {#1} \right)}
\newcommand{\pred}[1]{\left\llbracket \strut #1 \right\rrbracket}
\newcommand{\logg}[1]{\ln\!\left( \strut {#1} \right)}
\newcommand{\erm}{{\text{\tiny ERM}}}
\renewcommand{\mid}{\hspace{1pt}\vert\hspace{1pt}}
\newcommand{\brisk}{$\beta$-Risk\xspace}
\newcommand{\vc}{{\textit{Vapnik--Chervonenkis}}}

\newcommand{\tr}{\textsf{T}}
\DeclareMathAlphabet{\mat}{OT1}{cmss}{bx}{n}
\usepackage{xspace}
\newcommand{\gp}{$\gamma$-Precision\xspace}
\newcommand{\sgn}{\mathop{\mathrm{sgn}}}
\newcommand{\smin}{{\text{\tiny min}}}

\firstpageno{1}

\begin{document}

\title{Rethinking Generalisation}

\author{\name Antonia Marcu \email am1g15@soton.ac.uk \\
       \addr Vision, Learning and Control\\
       University of Southampton\\
       Southampton, UK
       \AND
       \name Adam Pr\"ugel-Bennett \email apb@ecs.soton.ac.uk \\
       \addr Vision, Learning and Control\\
       University of Southampton\\
       Southampton, UK}

\editor{Kevin Murphy and Bernhard Sch{\"o}lkopf}

\maketitle

\begin{abstract}
In this paper, a new approach to computing the generalisation performance is presented that assumes the distribution of risks, $\rho(r)$, for a learning scenario is known.
From this, the expected error of a learning machine using empirical risk minimisation is computed for both classification and regression problems.  A critical quantity in determining the generalisation performance is the power-law behaviour of $\rho(r)$ around its minimum value---a quantity we call \textit{attunement}.  The distribution 
$\rho(r)$ is computed for the case of all Boolean functions and for the perceptron used in two different problem settings.  Initially a simplified analysis is presented where an independence assumption about the losses is made.  A more accurate analysis is carried out taking into account chance correlations in the training set.  This leads to corrections in the typical behaviour that is observed.
\end{abstract}


\section{Introduction}
\label{sec:intro}

In traditional statistical learning theory the role of the learning algorithm is to eliminate all rules that poorly explain the data. This process relies on the idea that rules with poor generalisation performance (high risk) will, with high probability, make errors on a sufficiently large randomly chosen training data set \citep{vc, pac, NIPS1988_154,  blumer1989learnability, haussler1992decision, vapnik1992principles}.  Thus, given a large enough training set, the only rules left are those with good generalisation behaviour.  To formalise this, we assume there is a set, $\mathcal{H}$, of rules or hypotheses (we use the terms interchangeably).  We consider a learning problem where the training data is drawn independently from some underlying distribution and we minimise some loss function associated with each training example and each rule.  The expected loss for a hypothesis, $h$, over this distribution of data we term the risk, $R_h$.  Thus, our objective is to choose a hypothesis with low risk.  For a given training set we can compute the total loss, $L_h$, for all data points for a particular hypothesis, $h$. This loss associated with the training set is known as the \textit{empirical risk}.  We assume that we have an algorithm capable of choosing a hypothesis from the set
\begin{equation}
    \mathcal{H}_\erm = \{h\in\mathcal{H} | \forall_{h'} \; L_h \leq L_{h'}\},
\end{equation}
i.e. the set of hypotheses with minimum loss on the training set.  This is known as \textit{empirical risk minimisation}.  In traditional statistical learning theory the aim is to find a bound on the number of training examples such that with overwhelming probabilities all $h\in\mathcal{H}_\erm$ will have a risk less than some $\epsilon$.  To obtain such a bound there needs to be a finite number of hypotheses, otherwise, there could still be a high-risk hypothesis which by chance did well on the particular training set we used.  In the case where the learning machine has a continuous parameter space (so that the dimensionality of the space is uncountably infinite), we consider the effective size of the hypothesis space to be the \vc{} (VC) dimension \citep{vc}.  This effective size or capacity lies at the heart of conventional statistical learning theory.  By limiting the capacity we can obtain stronger bounds on the generalisation performance.

In this paper, we challenge this traditional approach.  Eliminating all high risk hypotheses is, in our view, too stringent and often leads to weak bounds. It is an unnecessary condition as good generalisation can be achieved with high probability so long as the  vast majority of hypotheses in $\mathcal{H}_\erm$ have low risk.  Thus, provided there is no bias towards choosing high-risk machines we will still, with high probability, chose a low-risk machine.  In this scenario, the capacity plays a much more minor role, instead, we need to know the distribution of risks, $\rho(r)$, for a learning scenario. That is, we need to know the proportion of hypotheses with a certain risk.  As we will show, the asymptotic generalisation performance is determined by the power-law growth in $\rho(r)$ for small $r$; a quantity we term attunement.

 This new perspective solves an apparent paradox first pointed out in an influential paper by \cite{zhang2016understanding}.  They studied some of the most successful deep learning networks, such as AlexNet~\citep{krizhevsky2012imagenet} and the Inception network~\citep{szegedy2015going}. They conducted empirical experiments on CIFAR-10 \citep{cifar}, a 10-way image classification task consisting of 50\,000 training images, and on ImageNet~\citep{imagenet_cvpr09}, a 1000-way classification task with over one million training images.  In their experiments rather than provide the correct labels for the training examples they trained the network with randomly shuffled labels.  Interestingly they were still able to find a set of parameters that for CIFAR-10 perfectly classified all the training examples, while for ImageNet they found network instances with very low errors on the training examples. This shows that these machines have a huge capacity and suggests that even for these very large training sets there still exist sets of parameters that will have high risk, but low training errors.  These networks have such a large capacity that conventional statistical learning theory can provide no useful guarantee of generalisation performance.  Nevertheless, when trained on real data, these networks achieved state-of-the-art results at the time they were first introduced.  In our approach, this provides no contradiction.  If we consider the set of parameters that perform well on the training set, then an overwhelming proportion of those parameters corresponds to low risk hypotheses.  Indeed in \cite{zhang2016understanding} they found that it took longer to train a network with random labels suggesting that they had to search much more of the parameter space to find a machine with a small training error (but effectively no generalisation ability).

The basic idea of our approach is simple.  We assume that we are given a set of hypotheses with a given distribution of risks, $\rho(r)$.  We eliminate hypotheses that perform poorly on the training examples.  This will, with overwhelming probability, remove more hypotheses with high risk, thus the expected risk of hypotheses in $\mathcal{H}_\erm$ will decrease as the size of the training set is increased.  Although the idea is simple, computing the expected ERM risk \textit{exactly} from $\rho(r)$ alone cannot be done.  We would require information about the correlation between hypotheses that depends on other details of our learning algorithm.  However, by making an assumption about the independence of losses for the hypotheses we can obtain an approximation to the expected ERM risk from $\rho(r)$ alone.  We do this for the case of classification and regression in Sections~\ref{sec:framework} and~\ref{sec:regression} respectively.  In Section~\ref{sec:rho} we compute $\rho(r)$ for the case of all Boolean functions and for the perceptron (for two different data distributions). Section~\ref{sec:assumptions} derives an exact expression for the expected risk in the realisable case---this result is data set dependent.  In Section~\ref{sec:dataset}, we study this data set dependence and obtain a more accurate approximation for the generalisation performance.  We discuss the results in the Conclusions.  Technical aspects we leave for the appendices.  In Appendix~\ref{sec:appPAC} we derive a PAC-like bound.  Finally in Appendix~\ref{sec:appAsym} we show that the asymptotic behaviour is dominated by the power-law behaviour of $\rho(r)$ for small $r$.

The work we present is closely related to work in the statistical mechanics literature on learning~\citep{engel2001statistical}.  This developed out of a seminal paper by Elizabeth Gardner which calculates the expected generalisation performance for a perceptron~\citep{Gardner_1988}.  That paper is technically challenging using replica methods but is believed to be exact in the limit when the number of features becomes infinite.  In this paper we attempt to develop a more general framework, although we are forced to use approximations.  The approximation developed in the next section is equivalent to the \textit{annealed approximation} in statistical mechanics.  In Section~\ref{sec:dataset} the corrections we obtain give very similar asymptotic generalisation results for the perceptron as those obtained by the replica calculation.  A similar approach to ours has also been put forward by \cite{scheffer1999expected}, however, rather than introducing a new framework for reasoning about generalisation, they used this approach to propose a model selection algorithm. For a number of classes of hypotheses, they estimate empirically the distribution of error rates from which the expected error of an ERM hypothesis from each class can be obtained. Although for their calculation \cite{scheffer1999expected} assumed independence of the losses, the implications of this assumption were not further investigated.

\section{Framework}
\label{sec:framework}

We seek to obtain stronger bounds than those obtained by considering the capacity of a learning machine.  To do so we require more information about the learning problem than in conventional learning theory.  In particular we assume that we are given a problem with a fixed dataset for which we know the distribution of risks, $\rho(r)$.  In Section~\ref{sec:rho}, we look at how to compute $\rho(r)$ for some specific problems. In general this is a hard task, nevertheless, by examining how the generalisation performance depends on $\rho(r)$, we can obtain a better understanding of what is required to improve generalisation.

Following conventional theory, we imagine that we are given a training data set of size $m$, where each training example is drawn at random and independently from the distribution of data that defines our problem.  We model our learning machine by a set, $\mathcal{H}$ of hypotheses. In our formalism the set of hypotheses may be finite or infinite. Each hypothesis, $h\in\mathcal{H}$, will have a loss associated with it.  In this section, we consider classification problems where we take the loss function to be 1 if we make a misclassification and 0 otherwise.  As each training data point is sampled independently, for any hypothesis, $h$, with risk $R_h$, the loss, $L_h$, will be binomially distributed
\begin{align*}
  \Prob{L_h=\ell|R_h} = \mathrm{Binom}(\ell|m,R_h) = \binom{m}{\ell}
  \, R_h^\ell \, (1-R_h)^{m-\ell}.
\end{align*}
However, the hypotheses may be correlated.  In this and the next section we ignore this correlation.  This makes the analysis relatively straightforward.  This approximation is equivalent to the annealed approximation in statistical mechanics~\citep{engel2001statistical}. In Sections~\ref{sec:assumptions} and~\ref{sec:dataset} we analyse the realisable case (where, at least one hypothesis correctly classifies all training examples) and show that chance correlations between training examples lead to a systematic correction in the typically observed generalisation performance.  There we show that the `annealed approximation' gives a reasonable qualitative description of the generalisation performances but is overly pessimistic.

As mentioned above, we consider the scenario known as \textit{empirical risk minimisation} where we choose a hypothesis from $\mathcal{H}_\erm$.  Importantly, we assume that every hypothesis in $\mathcal{H}_\erm$ is equally likely to be chosen.  This is sometimes referred to as \textit{Gibb's learning}.  As we discuss in Section~\ref{sec:dataset}, in the context of training a perceptron, Gibb's learning provides a good approximation to the performance of the perceptron learning algorithm.

Let $R_\erm\in\{R_h|h\in\mathcal{H}_\erm\}$  denote the risk of a randomly sampled hypothesis from the set of hypotheses with minimum loss on the training set.
Under the assumptions of our framework, the expected ERM risk is
\begin{align*}
  \av{R_\erm}
  &= \frac{\sum_{h\in\mathcal{H}} R_h \pred{h\in\mathcal{H}_\erm}}
  {\sum_{h\in\mathcal{H}} \pred{h\in\mathcal{H}_\erm}}
  = \sum_{\ell=0}^m \frac{\sum_{h\in\mathcal{H}} R_h \pred{L_h=\ell}}
    {\sum_{h\in\mathcal{H}} \pred{L_h=\ell}}\, \pred{\ell=L_\erm} \\
  &= \sum_{\ell=0}^m \av{r|\ell} \, \pred{\ell=L_\erm}.
\end{align*}
where $\pred{\text{predicate}}$ denotes an indicator function equal to 1 if the predicate is satisfied and 0 otherwise, and $L_\erm = \min \{L_h|h\in\mathcal{H}\}$ (i.e. the minimum empirical risk). 
Making the strong assumption that $\av{r|\ell} \approx \av[\mathcal{D}]{\av{r|\ell}}$ (i.e.{} that there are no significant fluctuations between data sets) then
\begin{align*}
  \av[\mathcal{D}]{\av{R_\erm}}
  \approx  \sum_{\ell=0}^m \av[\mathcal{D}]{\av{r|\ell}} \,
  \Prob{\ell=L_\erm}
\end{align*}
In the rest of this section we write $\av{\cdots} = \av[\mathcal{D}]{\av{\cdots}}$ (i.e. the expectation both with respect to the data set and over all hypotheses in $\mathcal{H}_\erm$).  Computing $\Prob{\ell=L_\erm}$ is inherently problematic in this formalism as it depends on the correlations between our hypotheses.  In the spirit of the approximation we are making we could ignore any correlation between hypotheses.  If we do this and define $f_L(\ell)=\Prob{L_h=\ell}$ for a random hypothesis $h\in\mathcal{H}$, and let $F_L(\ell)=\Prob{L_h\leq\ell}=\sum_{\ell'=0}^\ell f_L(\ell')$, then
\begin{align*}
  \Prob{\ell = L_\erm} & = \prod_{h \in \mathcal{H}} \Prob{L_{h}\geq\ell}
  - \prod_{h \in \mathcal{H}} \Prob{L_{h}\geq \ell+1}\\
  &= (1-F_L(\ell-1))^H - (1-F_L(\ell))^{H}
\end{align*}
where $H=|\mathcal{H}|$, that is, the size of the hypothesis space.  For an infinite hypothesis space, we should take $H$ to be some effective size of the hypothesis space (e.g. the $2^{D_{VC}}$, where $D_{VC}$ is the VC-dimension).  This is, the one area in our formalism where capacity plays an important role.  For realisable models $L_\erm$ is always equals 0 and we do not need to evoke capacity.

From Bayes' rule $f(r|\ell) = \Prob{\ell|r}\, \rho(r)/\Prob{\ell}$ so that
\begin{align*}
  \av{R|\ell} = \frac{ \int\limits_0^1 r  \,  \Prob{\ell|r} \, \rho(r) \, \dd r}
  { \int\limits_0^1  \Prob{\ell|r} \, \rho(r) \, \dd r}
  = \frac{ \int\limits_0^1 r ^{\ell+1}\,(1-r)^{m-\ell} \, \rho(r) \, \dd r}
  { \int\limits_0^1 r ^{\ell}\,(1-r)^{m-\ell} \, \rho(r) \, \dd r}
\end{align*}
Putting together the results above we obtain
\begin{align*}
  \av{R_\erm} = \sum_{\ell=0}^m
  \frac{ \int\limits_0^1 r ^{\ell+1}\,(1-r)^{m-\ell} \, \rho(r) \, \dd r}
  { \int\limits_0^1 r ^{\ell}\,(1-r)^{m-\ell} \, \rho(r) \, \dd r}
  \left((1-F_L(\ell-1))^H - (1-F_L(\ell))^{H}\right).
\end{align*}
In the realisable case (i.e. when there exists a hypothesis that
perfectly classifies all the training examples), then $L_\erm=0$ and
\begin{align*}
  \av{R_\erm} = \av{R|\ell=0}
  =  \frac{ \int\limits_0^1 r \,(1-r)^{m} \, \rho(r) \, \dd r}
  { \int\limits_0^1 (1-r)^{m} \, \rho(r) \, \dd r}.
\end{align*}

\subsection{Classification: \texorpdfstring{\brisk{}}{Lg} Model}
\label{sec:brisk}

We can numerically compute the expected ERM risk from a knowledge of the distributions of risks, $\rho(r)$.
In this section, we consider a special form of $\rho(r)$ that allows us to compute the integrals in closed form.
That is, we take $\rho(r)$ to be beta-distributed,
\begin{align}
  \rho(r) = \mathrm{Beta}(r|a,b) = \frac{r^{a-1}\,(1-r)^{b-1}}{B(a,b)}.
\end{align}
For a balanced data set where we perform a binary classification task we would choose $b=a$, while for k-way classification $b=a/(k-1)$ so that $\av{R_h}=(k-1)/k$.
Note that this distribution is unbiased, so, for example, in the binary case, there are as many poor hypotheses as good ones.
We call this the \brisk{} model.
The parameter $a$ measures the degree of ``attunement'': the smaller $a$ the more attuned is the hypothesis class $\mathcal{H}$ to the problem being solved.
The \brisk{} model allows us to obtain an intuitive understanding of the generalisation performance in this framework.
Although this seems a very particular functional form for $\rho(r)$, we show in Appendix~\ref{sec:appAsym} that for large $m$ the expected ERM risk is dominated by the power-law growth in $\rho(r)$, so that the \brisk{} model provides a reasonably accurate approximation for many different learning scenarios.
We explicitly compare the results obtained for the perceptron using the true $\rho(r)$ and a \brisk{} model with the same asymptotic behaviour in Section~\ref{sec:realperceptron}.

For the \brisk{} model the distribution of learning errors is given by
\begin{align}
  f_L(\ell) &= \av[R]{f(\ell|R)} = \binom{m}{\ell} \frac{B(a+\ell,
  b+m-\ell)}{B(a,b+m)}.
\end{align}
The conditional probability of a risk, $r$, given an empirical loss of $\ell$ is
\begin{align}
      f(r \mid \ell) = \frac{\Prob{\ell \mid r}\, f(r)}{\Prob{\ell}} =
  \frac{r^{\ell+a-1}\, (1-r)^{m-\ell+ b-1}}{B(\ell+a, m-\ell+ b)}. 
\end{align}
from which we find
\begin{align*}
  \av{R|\ell} = \frac{a+\ell}{m+a+b}.
\end{align*}
The \brisk{} model is a \textit{realisable problem} in the limit $H\rightarrow\infty$ since $\inf_{h\in\mathcal{H}} R_h=0$. 
That is, there exists a learning machine with arbitrarily small risk. 
In this case the expected ERM risk is $\av{R_\erm}=a/(a+b+m)$. 
We note in passing that if our learning algorithm does not return a hypothesis with the lowest possible empirical risk, but rather a hypothesis with a slightly higher empirical risk then, in the case where $a\gg1$ (which is typical), the expected risk of the returned hypothesis will not be significantly different from the expected ERM risk.  
That is, as is well known in practice, it is usually not that important to find a set of parameters that is guaranteed to minimise the empirical risk.

We can use the \brisk{} model to model unrealisable problems (i.e. when all hypotheses have a non-zero risk) by considering a finite hypothesis space.
In this case, there will be some best hypothesis with non-zero risk. 
For modelling finite-sized hypotheses spaces (a common abstraction in statistical learning theory) this is perfectly meaningful.
If we assume that our hypothesis space corresponds to samples drawn from a continuous parameter space of a learning machine then a non-realisable problem would be one where $\rho(r)=0$ for all $r<R_{\smin}$.
If we sample from $\rho(r)$ then all hypotheses will have a risk greater than or equal to $R_{\smin}$. 
To get a quick intuition about the generalisation behaviour for unrealisable problems, it is useful to consider the \brisk{} model with a finite hypothesis space. Figure~\ref{fig:expectRisk} shows the expected ERM risk versus $m$ plotted on a log-log scale for the case when $a=10^{2}$ and $a=10^{3}$ with different sized hypothesis spaces.

\begin{figure}[ht]
  \centering
  \includegraphics[width=0.6\linewidth]{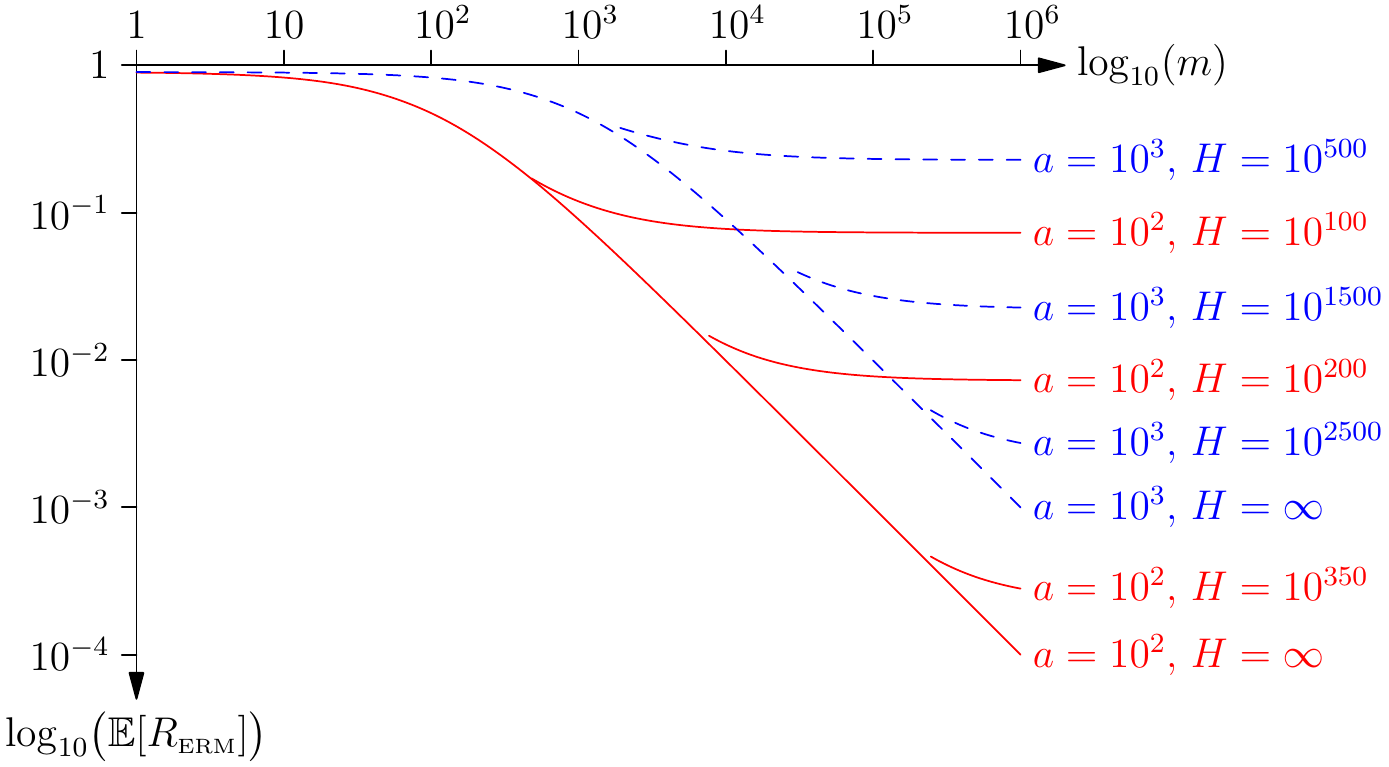}
  \caption{Expected ERM risk versus the number of training examples plotted
    on a log-log scale for $a=10^{2}$ (blue) and $a=10^{3}$ (red), with
    $b=a/9$ (i.e. for a 10 class problem) for different values of
    $H$.}
  \label{fig:expectRisk}
\end{figure}

We see in Figure~\ref{fig:expectRisk} that, for a given $a$, we can obtain better results for larger hypothesis spaces.  
This is because larger hypothesis spaces are likely to include lower risk hypotheses. 
Of course, if we use a richer, more complex
learning machine that increases the size of the hypothesis space, it is likely the machine would have worse attunement.

In standard statistical learning theory (which provides bounds on the asymptotic behaviour), there is a strong distinction made between realisable and non-realisable learning scenarios (whether or not the true concept exists in the set of hypotheses).  
In our framework, we observe that there is a zero-loss phase and a nonzero-loss phase in the expected ERM risk curves.  
For small $m$ some proportion of the learning machines are able to perfectly classify the training examples. 
If $\rho(r)$ is well approximated by a beta distribution around $\av{R_\erm}$ then $\av{R_\erm} \approx a/(a+b+m)$---this characterises the zero-loss phase.  When $\av{R_\erm}$ approaches the minimum risk $R_{\smin}$ (the risk of the best learning machine in $\mathcal{H}$) then $\av{R_\erm}$ will converge towards $R_{\text{min}}$. 
For realisable scenarios, $\av{R_\erm}$ will remain in the zero-loss phase for all $m$.

The typical bounds provided by statistical learning theory are on the number of training examples required to ensure a generalisation error of at most $\epsilon$ with a probability greater than $1-\delta$; these are known as \textit{Probably Approximately Correct} or PAC bounds \citep{pac}.
Classical PAC bounds in the realisable case depend on having a finite hypothesis space (or at least a finite capacity) as they require bounding the probability that all hypotheses with risk greater than $\epsilon$ have made at least one error on the training set with a probability of $1-\delta$. 
An analogous result in our framework would be to show that the ratio of hypotheses in $\mathcal{H}_\erm$ with risk greater than $\epsilon$ to those with risk less than $\epsilon$ is less than $\delta$. 
Technically, this is challenging to rigorously bound.   Under the assumption of an annealed approximation, we show in Appendix~\ref{sec:appPAC} that for the \brisk{} model, when $H\rightarrow\infty$ if the number of training examples satisfies \begin{align}\label{eq:pacOur}
  m \geq m^* = \frac{2\,a +2\, \ln(1/\delta)}{\epsilon} + 2\,a -2\,b +1,
\end{align}
then we will choose a machine whose risk is no greater than $\epsilon$ with a probability of at least $1-\delta$.
The annealed approximation appears to be overly conservative in which case this bound would still hold.
The equivalent bound  for a realisable learning problem from statistical learning theory \citep{pac} is
\begin{align}\label{eq:pac1}
  m \geq m^* = \frac{\ln(H) + \ln(\delta)}{\epsilon}.
\end{align}
This classical bound depends on the size of the hypothesis space. 
Our `bound' applies to hypothesis spaces of any size. 
For learning machines with continuous parameter spaces there exists a similar bound to Equation~(\ref{eq:pac1}), but with the VC-dimension playing a similar role to $\ln(H)$. 
The VC-dimension expresses the capacity of the model.
In our bound the role of $\ln(H)$ or the VC-dimension is played by the attunement parameter $a$.
This captures a quite different concept, namely how quickly does $\rho(r)$ fall off as $r\rightarrow 0$. 
If the learning machine is well attuned to the problem we would expect this to fall off relatively slowly.
Note that, whereas the capacity depends only on the learning machine architecture, the attunement also depends on the distribution of data.
We believe that the good performance of modern deep learning algorithms can be explained by their attunement. 
As pointed out in \cite{zhang2016understanding} the apparent vast VC dimension of deep learning machines renders the bound~(\ref{eq:pac1}) completely uninformative.

\section{Regression: \texorpdfstring{\gp{}}{Lg} Model}
\label{sec:regression}

To understand generalisation in the context of regression, we introduce an idealised problem setting, in which the problem is to find a function $h(\bm{x})$ to fit some true function $g(\bm{x})$ over a set of points $\mathcal{X}$. 
We take the loss function at a point $\bm{x}\in\mathcal{X}$ to be the squared error
\begin{align*}
  \ell(\bm{x}) = \epsilon_h^2(\bm{x})
  = \left( \strut h(\bm{x}) - g(\bm{x}) \right)^2.
\end{align*}
To characterise the performance of the function $h(\bm{x})$ we assume
that the set of errors $\epsilon_h^2(\bm{x})$ over $\mathcal{X}$ is
distributed according to
\begin{align*}
  f(\epsilon_h|\tau_h) = \sqrt{\frac{\tau_h}{2\,\pi}}\,
  \e{-\frac{\tau_h \, \epsilon_h^2(\bm{x})}{2}}
\end{align*}
where $\tau_h$ is a measure of the precision of the estimate $h(\bm{x})$.
Note that the risk, $R_h$, or expected loss, $R=\av{\ell(\bm{x})}$, is equal to $1/\tau_h$, so the higher the precision the better the fit.  
We will assume that the features $\bm{x}$ are high dimensional so that a typical set of training and test points will be relatively separated from each other. 
When we evaluate $h(\bm{x})$ at this set of data points, the errors, $\epsilon_h(\bm{x})$, can be treated as independent random variable drawn from $f(\epsilon_h|\tau_h)$.

We now introduce the \gp model where we assume that we have a countable set of hypotheses $\mathcal{H}$ with their precision drawn from a gamma distribution
\begin{align*}
  \tau_h \sim f_{\tau}(\tau)
  = \frac{b^a\,\tau^{a-1} \e{-b\,\tau}}{\Gamma(a)}.
\end{align*}
We note that rescaling $\tau$ corresponds to rescaling the functions $h(\bm{x})$ and $g(\bm{x})$ by a factor $\sqrt{\tau}$.  Since such a rescaling only changes the absolute size of the loss, but not the relative sizes of the loss, it makes no difference to the problem of selecting the best function.
As a consequence, we lose no generality by taking $b=a$.  
In this case the mean value of $\tau$ is 1, and the expected error over all points $\bm{x}$ and all hypotheses $h(\bm{x}) \in \mathcal{H}$ is $a/(a-1)$.  
The variance in $\tau$ is given by $1/a$.
This is a measure of attunement of the learning machine to the problem, where small $a$ indicates better attunement---there exists a higher proportion of hypotheses with high precision and consequently low risk.

We now assume that we have a training set $\mathcal{D} = \{(\bm{x}_i, y_i)| i=1,\ldots, m\}$ where $y_i=g(\bm{x}_i)$.
Scaling by half for mathematical convenience (it does not change the expected ERM risk), we define the empirical loss to be
\begin{align*}
  L_h = \frac{1}{2} \sum_{i=1}^m \epsilon_h^2(\bm{x}_i).
\end{align*}
A straightforward calculation shows that for this model the distribution of losses given the model precision of $\tau_h$ is
\begin{align*}
  f_L(L|\tau_h) = \frac{\tau_h^{\tfrac{m}{2}} \, L^{\tfrac{m}{2}-1}
  \e{-\tau_h\,L}}{\Gamma(\tfrac{m}{2})},
\end{align*} 
from which we find
\begin{align*}
  f_L(L) &= \int_0^\infty f_L(L|\tau) \, f_{\tau}(\tau)\,\dd \tau
  \\
  &= \frac{a^{a}\,L^{\tfrac{m}{2}-1}}{\Gamma(\tfrac{m}{2})\,\Gamma(a)} \int_0^\infty
    \tau^{a+\tfrac{m}{2}-1}\, \e{-(L+a)\,\tau} \dd \tau
    = \frac{1}{a\,B(a,\tfrac{m}{2})} \frac{(L/a)^{\tfrac{m}{2}-1}}{(1+L/a)^{a+m/2}}.
\end{align*}
Let $L_\erm$ be the loss with the smallest empirical loss, then in the
\gp model the distribution for $L_\erm$ is given by
\begin{align*}
  f_{L_\erm}(L) = H\,f_L(L)\, (1-F_L(L))^{H-1},
\end{align*}
where $F_L(L)$ is the cumulative probability function
\begin{align*}
  F_L(L) = \int_0^L f_L(L) \, \dd L.
\end{align*}
From which it follows that the expected ERM risk is
\begin{align*}
  \av{R_\erm} = \int_0^\infty \frac{a+L}{a+\tfrac{m}{2}-1} \, H\,f_L(L)\,
  (1-F_L(L))^{H-1} \, \dd L.
\end{align*}
We observe that for large $H$, if $m$ is sufficiently small so
that $L_\erm\approx 0$ then
\begin{align*}
  R_\erm \approx \frac{a}{a + \tfrac{m}{2} - 1}.
\end{align*}
The expected risk curve for the \gp model is shown in
Figure~\ref{fig:gammaPrecision} for $a=10^2$ and $a=10^3$ for different sizes of hypothesis space.  
We note that the curves are qualitatively remarkably similar to those for the \brisk model.

\begin{figure}[htbp!]
  \centering
  \includegraphics[width=0.65\linewidth]{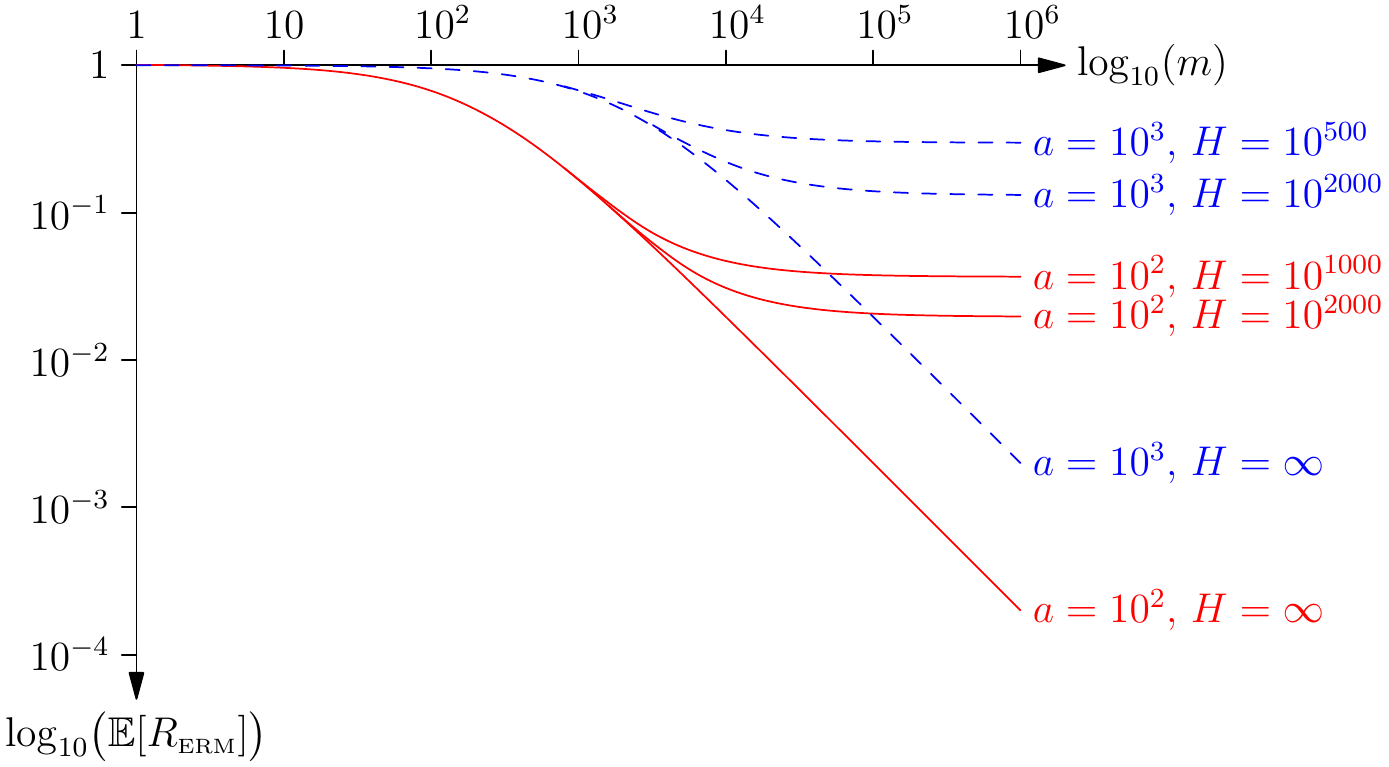}
  \caption{Expected ERM risk of the minimum loss model as a function of the
    number of training examples for $a=10^2$ and $a=10^3$ for different
    values of $H$.}
  \label{fig:gammaPrecision}
\end{figure}

The observation that good problem attunement is central to obtaining a low expected risk is consistent across the two classic machine learning problem settings studied in this paper.
In the following section, we consider cases in which capacity was traditionally invoked to explain generalisation.
We show that, in the light of our model, it is their attunement, captured by parameter $a$, that reflects prediction quality.

\section{The Distribution of Risks}
\label{sec:rho}

Key to our formalism is the need to know the distribution of risks, $\rho(r)$, for a learning problem. 
In this section, we compute $\rho(r)$ for three problems; the problem where we have a hypothesis space that includes all binary functions, a realisable perceptron and an unrealisable perceptron.

\subsection{All Binary Functions}

If the hypothesis space, $\mathcal{H}$, consists of all Boolean functions, $f:\mathcal{X}\rightarrow \{T,F\}$, where $\mathcal{X}$ is the set of all possible inputs, then the probability distribution of the risks for a randomly chosen hypothesis is given by
\begin{align}
    \rho(r) = \Prob{E=r\,N}
    = \mathrm{Binom}\!\left(E\,\Big|\, 2^{|\mathcal{X}|}, \tfrac{1}{2}\right)
    = \frac{1}{2^{2^{|\mathcal{X}|}}} \binom{2^{|\mathcal{X}|}}{E}
\end{align}
where $E$ is the number of errors made by the hypothesis.
In most machine learning applications $|\mathcal{X}|$ is exponential in the number of features.
For example, for binary strings of length $n$, $|\mathcal{X}|=2^n$. 
This distribution is very sharply concentrated around the mean $\av{R}=1/2$, having a variance of $|\mathcal{X}|/4$.  
We can approximate this distribution with a beta distribution where $a=b=|\mathcal{X}|/2$, which has the same mean and almost identical variance as the binomial distribution.\footnote{Recall for a beta distribution,
  $\mathrm{Beta}(r|a,b)$, that the mean is $a/(a+b)$ and the variance is equal to
  $a\,b/\left(\strut(a+b)^2(a+b+1)\right)$. So for $a=b$ the mean is
  $\tfrac{1}{2}$ and the variance is $1/(8a+4)$.}  
The expected ERM error for the beta distribution approximation is $|\mathcal{X}|/(2\, |\mathcal{X}|+ m)$.
We therefore require $m$ to be of order $|\mathcal{X}|$ before we obtain any generalisation performance (of course, in this problem we can only generalise on the data points we have seen).

In this case, the lack of generalisation is a result of the huge value of the attunement parameter rather than the size of the hypothesis space. 
Of course, for a binary problem, a hypothesis space consisting of all binary functions is as large as it can be. 
To achieve generalisation we require a smaller hypothesis space.
However, as we have demonstrated, we can achieve good generalisation even for hypothesis spaces large enough that we can, with high probability, find a dichotomy for a large number of training patterns.  
From the experiments of Zhang \textit{et al.}~(2016), the fact that we can learn the set of 50\,000 training images with random labels of 10 classes suggest a hypothesis space consisting of at least $10^{50\,000}$ hypotheses.  
However, this is much smaller than $2^{|\mathcal{X}|}$, which for colour images with 1K pixels taking 256 values is $2^{256^{3072}}$.  
Provided $|\mathcal{H}|$ is substantially smaller than this, we can still achieve a relatively high degree of attunement (i.e. small value of $a$).

The simple problem of learning all binary functions illustrates a case of poor attunement, which leads to no generalisation. 
Below, we study the case of a well-attuned perceptron. 
We calculate its risk probability density and relate back to our \brisk{} model to analyse changes in attunement as a result of feature reduction.

\subsection{Realisable Perceptron}
\label{sec:realperceptron}

We consider a very simple learning scenario with data set consisting of pairs $(\bm{x},y)$ where $y = \sgn(\bm{x}^\tr\bm{w}^*)$ and where $\bm{w}^*$ is a $p$-dimensional randomly chosen vector with unit norm.  That is, $y=1$ if the data is positively correlated with the target vector $\bm{w}^*$ that defines the separating plane and $y=-1$ otherwise.
We further assume that $\bm{x}$ is distributed according to a normal distribution $\mathcal{N}(\bm{x}|\bm{0},\mat{I})$.  
Our training set corresponds to $m$ pairs $(\bm{x}_i,y_i)$ where $\bm{x}_i\sim \mathcal{N}(\bm{0},\mat{I})$ and $y_i = \sgn(\bm{x}_i^\tr\bm{w}^*)$. 
We consider learning this with a perceptron with weights $\bm{w}$ such that $\|\bm{w}\|=1$.

If we consider sampling uniformly from the set of weight vectors then the distribution of weight vectors with an angle $\theta$ to $\bm{w}^*$ is \begin{align} \label{eq:weightDist}
   f_\Theta(\theta) = \frac{\sin^{p-2}(\theta)}{B(\tfrac{1}{2}, \tfrac{p-1}{2})}.
\end{align}
For this problem the risk is given by $r=\theta/\pi$ so that $\rho(r) = \pi\,f_\Theta(\pi\, r)$.
This is a realisable model for which the expected ERM risk, under the assumption of the annealed approximation, is
\begin{align*}
    \av{R_\erm} = \av{R|\ell=0}
  =  \frac{ \int\limits_0^1 r \,(1-r)^{m} \, \sin^{p-2}(\pi\,r) \, \dd r}
  { \int\limits_0^1 (1-r)^{m} \, \sin^{p-2}(\pi\,r) \, \dd r}.
\end{align*}
We can compute this numerically, however, when $m$ is large the dominant contribution to the integral comes from where $r$ is small.  
In this region $\rho(\pi\,r)$ grows as $r^{p-2}$ (since $\sin(\pi\,r)$ grows linearly with $r$ for small $r$).  
Thus we can approximate $\rho(r)$ by a beta distribution $\mathrm{Beta}(r|p-1,p-1)$ for which $\av{R_\erm} = (p-1)/(2\,p-2+m)$.

In Figure~\ref{fig:realPerc}, we show $\av{R_\erm}$ as a function of the number of training examples, $m$, for the realisable perceptron (computed numerically) and the \brisk{} model with $a=b=p-1$.  
We see that the \brisk{} model provides a good approximation to the realisable perceptron in the annealed approximation.
In Section~\ref{sec:dataset} we obtain corrections to the annealed approximation.  These corrections change the $(1-r)^m$ term rather than the $\rho(r)$ term.

\begin{figure}[htbp]
  \centering
  \includegraphics[width=0.6\linewidth]{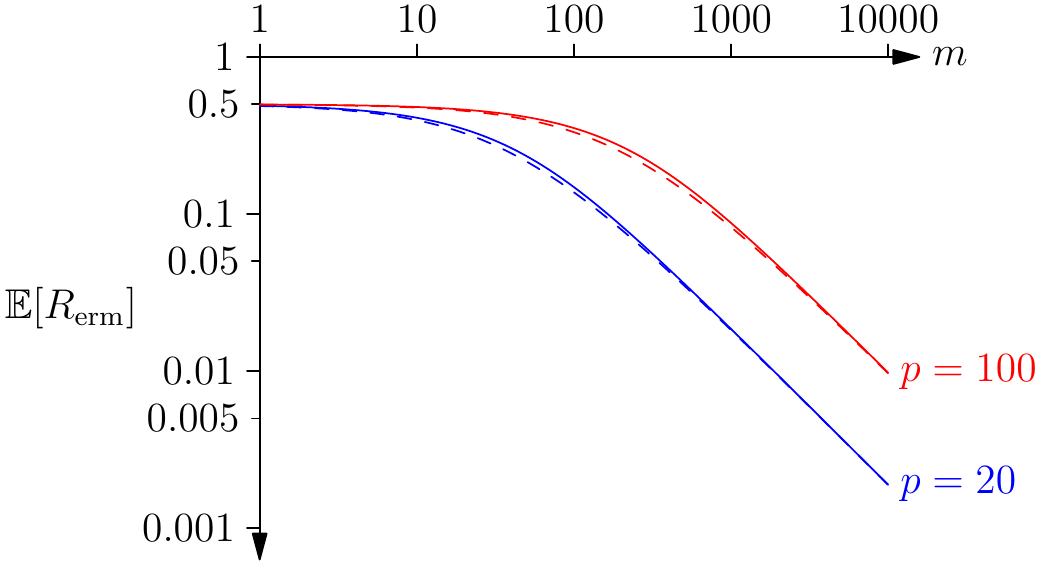}
  \caption{Expected ERM risk for the realisable perceptron for the cases
    $p=20$ and $p=100$.   We also show as dashed curves are $\av{R_\erm}
    =(p-1)/(2\,p-2+m)$ corresponding to a \brisk{} model with $a=p-1$.}
  \label{fig:realPerc}
\end{figure}

For this simple scenario, the distribution of risks (and hence the attunement) is directly determined by the dimensionality of the vector $\bm{w}^*$.  
If $\bm{w}^*$ is orthogonal to some of the features, then they can be removed, improving generalisation.  
Traditionally, this would be attributed to reducing the size of the hypothesis space.
However, we see that this also leads to an improvement in the attunement (compare the solid curves in Figure~\ref{fig:realPerc}, indicating the improvement we would expect if starting from $p=100$ features we could remove 80 features that did not affect the risk).

\subsection{Unrealisable Perceptron}

We now consider using a perceptron with a different distribution of
data.  Consider a two-class problem with data
$(\bm{x},y)$ where $y\in\{-1,1\}$ and $\bm{x}$ is
\begin{align*}
  f_X(\bm{x}|y)
  = \mathcal{N}\!\left(\bm{x}\big|\Delta\,y\,\bm{w}^*,\mat{I}\right),
\end{align*}
where $\bm{w}^*$ is some arbitrary unit norm vector.  
The parameter $\Delta$ determines the separation between the means of the two classes.
The Bayes optimal classifier corresponds to a hyperplane orthogonal to $\bm{w}^*$.
We consider learning a perceptron defined by the unit variance weight vector $\bm{w}$.  An elementary calculation shows that
\begin{align*}
    y\,\bm{x}^\tr \, \bm{w} = \Delta\,\cos(\theta) + \eta,
\end{align*}
where
$\eta = y\,\bm{w}^\tr(\bm{x}-y\,\Delta,\bm{s}^*) \sim \mathcal{N}(0,1)$
and $\cos(\theta) = \bm{w}^\tr \bm{w}^*$.  The expected risk is
\begin{align*}
  R_{\bm{w}} = \Prob{y\,\bm{x}^\tr \, \bm{w}<0}
  = \Prob{\Delta\,\cos(\theta) < -\eta} = \Phi(-\Delta\,\cos(\theta))
\end{align*}
where $\Phi(z)$ is the cumulative probability distribution for a zero mean, unit variance normally distributed random variable.  
The distribution of weight vectors at an angle $\theta$ to $\bm{w}^*$ is the same as that for the realisable perceptron (Equation~(\ref{eq:weightDist})).
The distribution of risks is given by $f_R(r) = f_\Theta(\theta(r))/\tfrac{\dd r}{\dd \theta}$, where $r = \Phi(-\Delta\,\cos(\theta))$ or
$\theta(r) = \arccos(\Phi^{-1}(r)/\Delta)$.
Noting that
\begin{align*}
  \frac{\dd r}{\dd \theta}
  = \Delta\, \sin(\theta) \, \frac{\e{-\Delta^2\,\cos^2(\theta)/2}}{\sqrt{2\,\pi}}
\end{align*}
and writing
\begin{align*}
    \sin^{p-3}(\theta) =  \left( 1-\cos^2(\theta) \right)^{\frac{p-3}{2}} = \left( 1-\left(\frac{\Phi^{-1}(r)}{\Delta}\right)^2 \right)^{\frac{p-3}{2}}
\end{align*}
we get
\begin{align*}
    \rho(r) = \frac{\sqrt{2\,\pi}}{\Delta\, B(\tfrac{1}{2},
  \tfrac{p-1}{2})} \,
  \left( 1-\left(\frac{\Phi^{-1}(r)}{\Delta}\right)^2
  \right)^{\frac{p-3}{2}}
  \, \e{(\Phi^{-1}(r))^2/2}.
\end{align*}

To help understand this equation, in Figure~\ref{fig:perceptron} we depict the probability density, $\rho(r)$, plotted against the risk, $r$, on a logarithmic scale for two different levels of class separability which correspond to (\ref{fig:perceptron}.a) $R_\smin=0.25$ and (\ref{fig:perceptron}.b) $R_\smin=0.001$.
For each of them, we look at varying the number of features.
 
\begin{figure}[htbp]
  \centering
  \includegraphics[width=0.37\linewidth]{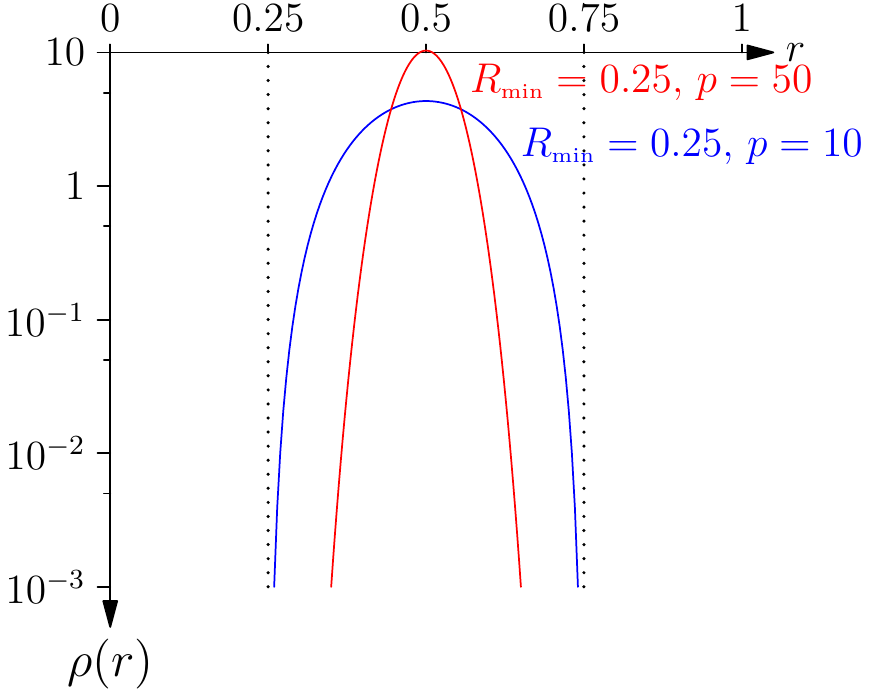} \hfil
  \includegraphics[width=0.37\linewidth]{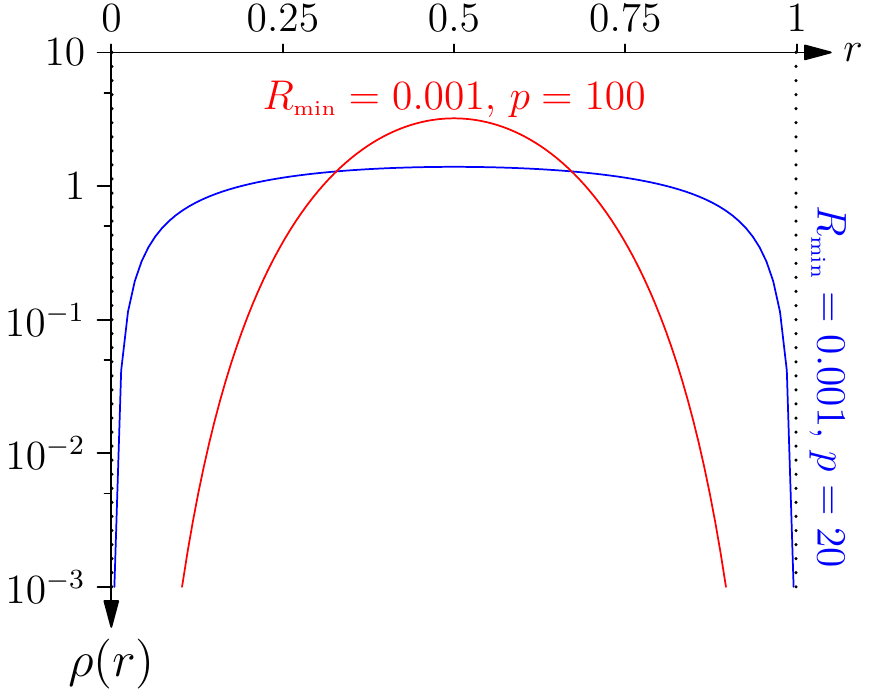}\\
  (a) \hfil \hfil (b)
  \caption{Probability density, $f_R(r)$, plotted on a logarithmic
      scale against the risk, $r$, for (a)
      $R^*=0.25$ so that $\Delta\approx-0.674$ with $p=10$ and $p=50$,
      and (b) $R=0.001$ so that
      $\Delta\approx-3.090$ with $p=20$ and $p=100$.  The vertical
      dotted lines show the maximum and minimum risks in $\mathcal{H}$.}
    \label{fig:perceptron}
\end{figure}

We note that for unrealisable models the distribution of risks, $\rho(r)$, will be 0 for $r<R_\smin$.  When $\av{R_\erm}$ is substantially greater than $R_\smin$, then the generalisation behaviour will be similar to a realisable model with the same attunement. 
As $m$ increases, $\av{R_\erm}$ will converge to $R_\smin$. 
The two quantities that characterise the asymptotic behaviour in the unrealisable case are $R_\smin$ and the power-law growth of $\rho(r)$ as we increase $r$ from $R_\smin$.

\section{Revisiting Assumptions}
\label{sec:assumptions}

In our analysis, we made some cavalier assumptions, in particular about the independence of the losses.  
We also replaced the expectation of a ratio by the ratio of expectations. 
This is clearly only an accurate approximation if the values are heavily concentrated around their mean. 
In this section, we treat these assumptions and approximations more carefully.

We assume that we have a realisable problem for which the expected risk for $h\in \mathcal{H}_\erm$ for a given training set, $\mathcal{D}$, is
\begin{align*}
  \av{R_\erm|\mathcal{D}} =
  \av[\mathcal{H}]{
  \frac{\sum\limits_{h\in\mathcal{H}} r_h \, \chi_h}
  {\sum\limits_{h'\in\mathcal{H}} \chi_{h'}} }
\end{align*}
where $\chi_h=\pred{h\in \mathcal{H}_\erm}$ is an indicator function equal to 1 if $h\in \mathcal{H}_\erm$ and 0 otherwise.  
Clearly, $\chi_h$ will depend on $\mathcal{D}$.

To model our learning machine we consider constructing a countable set of hypotheses, $\mathcal{H}$, by randomly sampling hypotheses from the learning machine (for learning machines with continuous parameters we would uniformly at random choose a set of parameters).  In the limit $H\rightarrow\infty$ we would expect that $\mathcal{H}$ will have the same generalisation properties as the true machine (we would not expect replacing a continuous learning space by a discrete set of points to change the performance of the machine; after all, we simulate continuous machines on a computer using finite precision arithmetic).
For any given set of data, we note that, since we are randomly sampling our parameter space to obtain a sequence of hypotheses, $h_1$, $h_2$, \ldots, then the distribution of the Bernoulli variables $\chi_{h}$ will be interchangeable.  
That is, if $\pi$ is a permutation of the indexes, then for any $H\in\mathbb{N}$
\begin{align*}
  \Prob{\chi_1,\chi_2,\ldots,\chi_H} =
  \Prob{\chi_{\pi(1)},\chi_{\pi(2)},\ldots,\chi_{\pi(H)}}.
\end{align*}
By de Finetti's theorem the random variables are independent and
identically distributed conditioned on
\begin{align*}
  M_0 = \av[h]{\chi_h|\mathcal{D}},
\end{align*}
where the expectation is over drawing sample hypotheses.  
$M_0$ will fluctuate between training sets.

For a given training set such that $M_0 = \av[h]{\chi_h|\mathcal{D}}$ we can treat the $\chi_h$'s as independent.  
We denote the cardinality of $\mathcal{H}_\erm$ by $S=|\mathcal{H}_\erm|$, then
\begin{align}\label{eq:RermReal}
  \av{R_\erm|\mathcal{D}} =
  \av[\bm{\chi}]{\sum_{S=1}^H \frac{1}{S}
  \sum\limits_{h\in\mathcal{H}} r_h \, \chi_h
  \pred{\sum_{h'\in\mathcal{H}} \chi_{h'} =S} }.
\end{align}
Using the integral representation for the indicator function
\begin{align*}
  \pred{\sum_{h'\in\mathcal{H}}  \chi_{h'} =S}
  = \int_0^{2\pi} \e{-\ii\,\omega \left(S-\sum\limits_{h'\in\mathcal{H}}
  \chi_{h'}\right)}
  \, \frac{\dd \omega}{2\pi}
\end{align*}
and writing
\begin{align*}
  \e{\ii\,\omega \sum\limits_{h'\in\mathcal{H}}  \chi_{h'}}
  = \prod_{h'\in\mathcal{H}} \left(\strut \chi_{h'}(\e{\ii\,\omega}-1) + 1\right)
\end{align*}
(where we used the fact that $\chi_{h'}\in\{0,1\}$), then we can rewrite
Equation~(\ref{eq:RermReal}) as
\begin{align*}
  \av{R_\erm|\mathcal{D}} =
  \av[\bm{\chi}]{\sum_{S=1}^H \frac{1}{S}
  \int_0^{2\pi} \e{-\ii\,\omega S}
  \sum_{h\in\mathcal{H}} r_h \, \chi_h
    \prod_{h'\in\mathcal{H}} \left(\strut \chi_{h'}(\e{\ii\,\omega}-1) +
  1\right)\, \frac{\dd \omega}{2\pi} }.
\end{align*}
Since $\chi_h\in\{0,1\}$ (so that $\chi_h^2=\chi_h$), then
\begin{align*}
  \av{R_\erm|\mathcal{D}} =
  \av[\bm{\chi}]{
  \sum_{S=1}^H \frac{1}{S}
  \int_0^{2\pi} \e{-\ii\,\omega S}
  \sum_{h\in\mathcal{H}} r_h \, \chi_h\,\e{\ii\,\omega}
  \prod_{\genfrac{}{}{0pt}{}{h'\in\mathcal{H}}{h'\neq h}}
  \left(\strut \chi_{h'}(\e{\ii\,\omega}-1) + 1\right)
  \, \frac{\dd \omega}{2\pi} }.
\end{align*}

We note that for our training set
$\av[\bm{\chi}]{\chi_{h'}|\mathcal{D}} = M_0$.  We define
$M_1 = \av[\bm{\chi}]{r_h\,\chi_h|\mathcal{D}}$.  Thus, since the
$\chi_h$'s are all IID distributed, then
\begin{align*}
  \av{R_\erm} &=
  H \sum_{S=1}^H \frac{1}{S}
  \int_0^{2\pi} \e{-\ii\,\omega S}
   M_1\,\e{\ii\,\omega}
  \left(\strut M_0(\e{\ii\,\omega}-1) + 1\right)^{H-1}
                \, \frac{\dd \omega}{2\pi}.
\end{align*}
Using the binomial expansion of $\left(M_0\,\e{\ii\,\omega} + (1-M_0)\right)^{H-1}$,
\begin{align*}
  \av{R_\erm|\mathcal{D}}
  &= H\, M_1 \sum_{S=1}^H \frac{1}{S} \int_0^{2\pi}
     \e{-\ii\,\omega S}\e{\ii\,\omega} \sum_{k=0}^{H-1}\binom{H-1}{k}
     M_0^k\e{\ii\,\omega k} (1-M_0)^{H-1-k}
     \, \frac{\dd \omega}{2\pi} \\
  &=H \, M_1 \sum_{S=1}^H \frac{1}{S}  
    \sum_{k=0}^{H-1} \binom{H-1}{k}
    M_0^k\int_0^{2\pi} \e{\ii\,\omega (k+1-S)} \,
    \frac{\dd \omega}{2\pi} (1-M_0)^{H-1-k}\\
  &=H \, M_1 \sum_{S=1}^H \frac{1}{S}  
    \sum_{k=0}^{H-1} \binom{H-1}{k}
    M_0^k \pred{k=S-1} (1-M_0)^{H-1-k}\\
  &= H\, M_1 \sum_{S=1}^H \frac{1}{S}  \binom{H-1}{S-1}
    M_0^{S-1} (1-M_0)^{H-S}.
\end{align*}
Given that we are dealing with a sum of IID Bernoulli variables it
should be no surprise that we end up with a binomial distribution.  This
could have probably been written down immediately from
Equation~(\ref{eq:RermReal}), but to ensure all terms are correct we
prefer a purely algebraic derivation.  We note that
\begin{align*}
  \frac{1}{S}  \binom{H-1}{S-1} = \frac{1}{H} \binom{H}{S},
\end{align*}
so that
\begin{align}\label{eq:exactRerm}
  \av{R_\erm|\mathcal{D}}
  &= \frac{M_1}{M_0} \sum_{S=1}^H  \binom{H}{S}
    M_0^{S} (1-M_0)^{H-S} \\
  &= \frac{M_1}{M_0} \left(1-(1-M_0)^{H}\right)
    \approx \frac{M_1}{M_0} \left(1-\e{-H\,M_0}\right),
\end{align}
where we used the fact that the terms in the sum are equal to $\mathrm{Binom}(S|H,M_0)$, so summing over $S$ from 0 to $H$ will give 1.
The sum, however, misses the first term which leads to the correction term  $(1-M_0)^{H}$.  $H\,M_0$ is the number of hypotheses that correctly satisfy all the training examples.  
In the limit $H\rightarrow\infty$, the term $\exp(-H\,M_0)$ is infinitesimal.  
Even for finite hypothesis spaces, this correction term will nearly always be negligibly small.  
We ignore this correction in the rest of the paper.

\section{Corrections due to Fluctuations}
\label{sec:dataset}

For any independently chosen finite data set, $\mathcal{D} = \{(\bm{x}^\mu, y^\mu) | \mu=1,\,2,\,\ldots,\,m\}$, there will be chance fluctuations between the features vectors, $\bm{x}^\mu$, that lead to variation in the generalisation performance. 
Perhaps more surprisingly, these fluctuations lead to a significant change in the mean behaviour.
In this section we derive an approximation to these corrections, which depend on the detail of the learning machine beyond the distribution of the risks, so cannot be computed in general.  We derive corrections for the realisable perceptron.
For any realisable model we have shown that
\begin{align*}
    \av{R_\erm|\mathcal{D}} = \frac{M_1}{M_0}
    = \frac{\av[r\sim\rho(r)]{r\,p(r,\mathcal{D})}}{\av[r\sim\rho(r)]{p(r,\mathcal{D})}}
\end{align*}
where $p(r,\mathcal{D})$ is the proportion of hypotheses with risk $r$ that correctly classify all training examples (i.e. $\forall\, (\bm{x},y)\in\mathcal{D}\; h(\bm{x})=y$, where $h(\bm{x})$ denotes the prediction of hypothesis $h$ given a feature vector $y$).
In Figure~\ref{fig:fluctuations} we illustrate schematically what $p(r,\mathcal{D})$ might look like for the realisable perceptron.

\begin{figure}[htbp]
  \centering
  \includegraphics[width=0.27\linewidth]{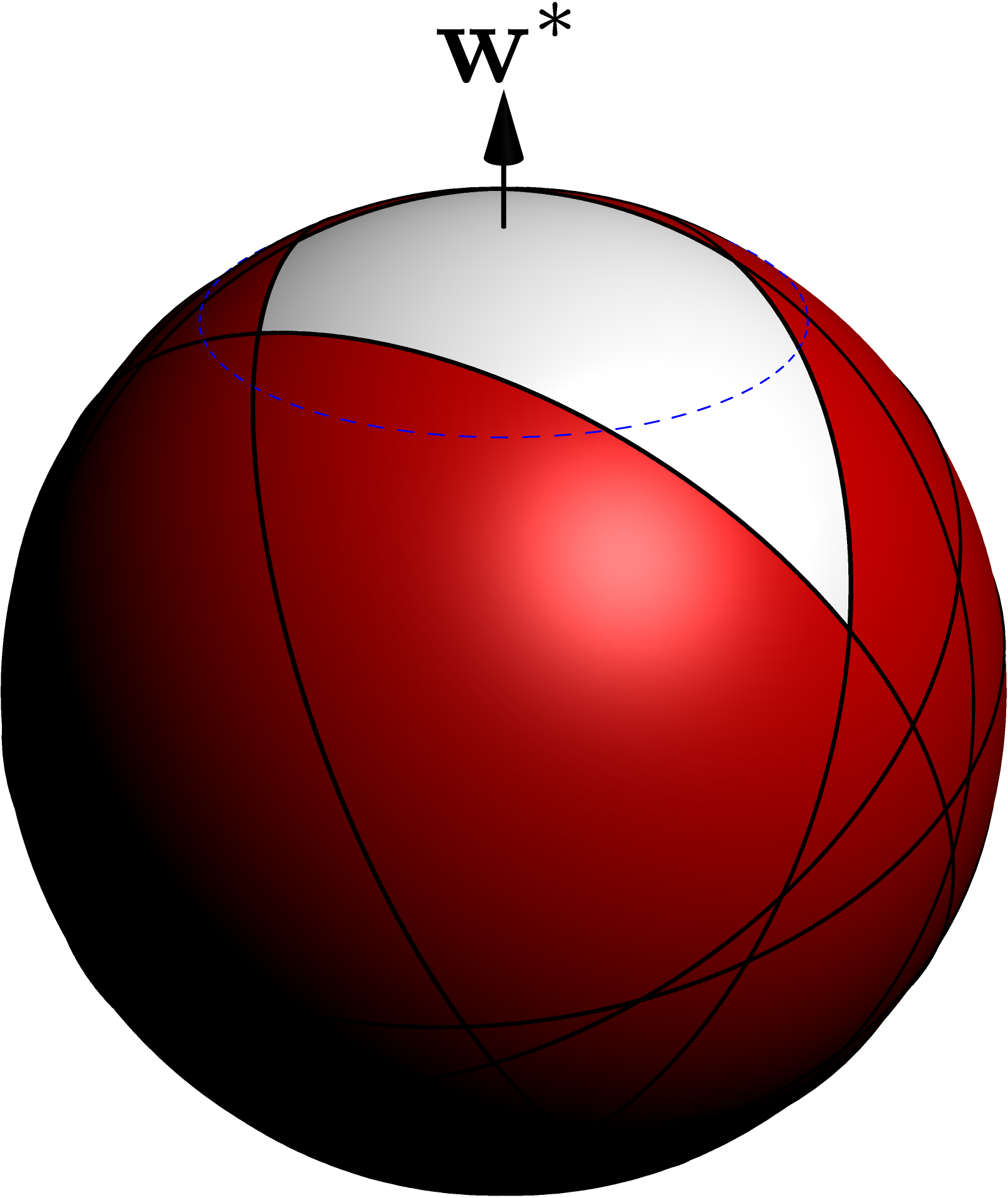} \hfil
  \includegraphics[width=0.35\linewidth]{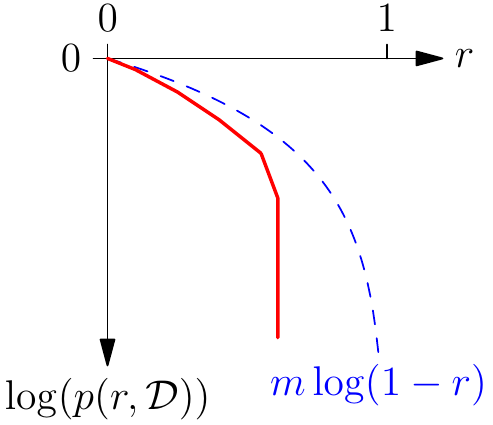}\\
  (a) \hfil \hfil (b)
  \caption{Schematic illustration of the proportion, $p(r,\mathcal{D})$
    of parameter space with risk~$r$ that correctly classifies the
    training examples for a realisable perceptron.  (a) Shows the weight
    space for a perceptron with three inputs.  The white area represents
    the parameters that correctly classify the examples.  The vector
    $\bm{w}^*$ represents the correct concept.  The proportion
    $p(r,\mathcal{D})$ would corresponds to the fraction of weight space
    at a given value of $r$ that correctly classifies the inputs.  For the perceptron, constant risk would correspond to line of constant latitude as illustrated by the dashed blue line.  
    (b) Shows an illustrate of $\log(p(r,\mathcal{D}))$ together with $\log(\av{p(r,\mathcal{D})})$ (dashed curve).}
  \label{fig:fluctuations}
\end{figure}

Denote the set of hypotheses with risk $r$ that correctly classify the first $k$ training examples by
\begin{align*}
    \mathcal{H}_r^k = \left\{ h\in\mathcal{H} \big| R_h =r \wedge \forall\, \mu=1,\,2,\,\ldots,\,k\; h(\bm{x}^\mu)=y^\mu \right\},
\end{align*}
then $p(r,\mathcal{D})= |\mathcal{H}_r^m|/|\mathcal{H}_r^0|$, where $\mathcal{H}_r^0$ is the set of hypotheses with risk $r$.  
We note that we can also write $p(r,\mathcal{D})$ as
\begin{align*}
p(r,\mathcal{D}) = \prod_{k=1}^m \frac{|\mathcal{H}_r^{k}|}{|\mathcal{H}_r^{k-1}|}
= \frac{|\mathcal{H}_r^{m}|}{|\mathcal{H}_r^{m-1}|}
\frac{|\mathcal{H}_r^{m-1}|}{|\mathcal{H}_r^{m-2}|} \ldots
\frac{|\mathcal{H}_r^{1}|}{|\mathcal{H}_r^{0}|}
= \frac{|\mathcal{H}_r^{m}|}{|\mathcal{H}_r^{0}|}.
\end{align*}
Defining $p_r^k = |\mathcal{H}_r^{k}| / |\mathcal{H}_r^{k-1}|$ then 
\begin{align*}
    p(r,\mathcal{D}) = \prod_{k=1}^m p_r^k.
\end{align*}

The quantity $p_r^k$ is the proportion of hypotheses in $\mathcal{H}_r^{k-1}$ that correctly classify the $k^{th}$ data point.
By the definition of risk, $\av{p_r^k}=1-r$.  However, due to chance correlations between training examples, $p_r^k$, will fluctuate.  As the training examples are drawn independently, $p_r^k$ and $p_r^j$ will be independent random variables when $k\neq j$.  Now
\begin{align*}
    \logg{p(r,\mathcal{D})} = \sum_{k=1}^m \logg{p_r^k}
\end{align*}
is a sum of independent random variables and by the central limit theorem this sum will converge towards a normal distribution ($m$ is usually sufficiently large that the distribution of $\logg{p(r,\mathcal{D})}$ will be very closely approximated by a normal distribution). In consequence, $p(r,\mathcal{D})$ will be close to a log-normal distribution and its median value will typically be smaller than its mean. The typical value of $\av{R_\erm|\mathcal{D}}$ is going be given when $p(r,\mathcal{D})$ takes its most likely value or equivalently by the median of $\logg{p(r,\mathcal{D})}$ .  Since $\logg{p(r,\mathcal{D})}$ is normally distributed, its median, mode and mean are all the same.  Thus to compute the typical value of $\av{R_\erm|\mathcal{D}}$ we can use the most likely value of $p(r,\mathcal{D})$ which will be
\begin{align*}
    \hat{p}(r,\mathcal{D}) = \exp\!\left( \av{\logg{p(r,\mathcal{D})}} \right)
\end{align*}
where
\begin{align*}
    \av{\logg{p(r,\mathcal{D})}} = \sum_{k=1}^m \av{\logg{p_r^k}}
    = \sum_{k=1}^m \av{\frac{|\mathcal{H}_r^{k}|}{|\mathcal{H}_r^{k-1}|}}.
\end{align*}
By Jensen's inequality $\av{\logg{p(r,\mathcal{D})}} \leq \logg{\av{p(r,\mathcal{D})}}$.  This does not tell us whether the fluctuations improve or worsen the generalisation performance (which depends on the gradient of $\logg{p(r,\mathcal{D})}$). However, for $r=0$ we know that $p(r,\mathcal{D})=1$ so that the fluctuations can only increase the gradient of $\av{\logg{p(r,\mathcal{D})}}$ around $r=0$.  As this gradient determines the asymptotic generalisation performance (what we have termed the attunement) we see that the `annealed approximation' will be an upper bound on the asymptotic generalisation performance (i.e. it will be overly conservative).

To get an understanding of the quantitative corrections we need to model the fluctations that we are likely to get in $p_r^k$.  As $p_r^k$ is a random variable that lies in the range from 0 to 1 it is reasonable to approximate its distribution by a beta distribution.  That is $p_r^k \sim \mathrm{Beta}(A_r^k,B_r^k)$ (this distribution is not related to that used in the \brisk{} model---we use a beta distribution as in both cases we are modelling a random variable that lies in the range 0 to 1).  If $p_r^k$ is beta distributed then
\begin{align*}
    \av{\logg{p(r,\mathcal{D})}} = \sum_{k=1}^m \left( \strut \psi(A_r^k)
        - \psi(A_r^k + B_r^k) \right)
\end{align*}
Using the fact that $\av{p_r^k}=1-r$ and denoting the variance of $p_r^k$ by $v_r^k$ then we find
\begin{align*}
    A_r^k &= \frac{1-r}{\Delta_r^k}, & B_r^k &= \frac{r}{\Delta_r^k}, &
    \Delta_r^k = \frac{v_r^k}{r(1-r)-v_r^k},
\end{align*}
By definition
\begin{align*}
    v_r^k = \av{(p_r^k)^2} - (1-r)^2
\end{align*}
where
\begin{align*}
    \av{(p_r^k)^2} &= \av{ \frac{|\mathcal{H}_r^{k}|^2}{|\mathcal{H}_r^{k-1}|^2} }
    \\
    &= \av{ \frac{1}{|\mathcal{H}_r^{k-1}|^2} 
    \sum_{h\in \mathcal{H}_r^{k-1}} \sum_{h'\in \mathcal{H}_r^{k-1}} \av[(\bm{x}^k,y^k)]{\pred{h(\bm{x}^k)=y^k}\,\pred{h'(\bm{x}^k)=y^k}} }.
\end{align*}
We observe that the fluctuations depend on the expected correlation between hypotheses of a given risk.  Denoting the joint probability of a pair of hypotheses making a particular prediction for a randomly sampled data-point, $(\bm{x},y)$, by
\begin{align*}
    \Prob{\pred{h(\bm{x})=y}=w,\, \pred{h'(\bm{x})=y}=z} = p_{wz}(h,h')
\end{align*}
(with $w,z\in\{0,1\}$) then
\begin{align*}
    \av{(p_r^k)^2} &= \frac{1}{|\mathcal{H}_r^{k-1}|^2} \sum_{h\in \mathcal{H}_r^{k-1}} \sum_{h'\in \mathcal{H}_r^{k-1}} p_{11}(h,h').
\end{align*}
We note that $p_{11}(h,h') + p_{01}(h,h') = \Prob{h(\bm{x})=y} = 1-r$.  Also for randomly selected hypotheses by symmetry $p_{10}(h,h') = p_{01}(h,h')$, while for any pair of hypotheses $p_{10}(h,h') + p_{0,1}(h,h')=\Prob{h(\bm{x})\neq h'(\bm{x})}$.  From this we find $p_{11}(h,h') = 1-r - \Prob{h(\bm{x})\neq h'(\bm{x})}/2$ and the variance in $p(r,\mathcal{D})$ is given by
\begin{align*}
    v_r^k = r(1-r) - \frac{1}{2\,|\mathcal{H}_r^{k-1}|^2} \sum_{h\in \mathcal{H}_r^{k-1}} \sum_{h'\in \mathcal{H}_r^{k-1}} \Prob{h(\bm{x})\neq h'(\bm{x})}.
\end{align*}

Up to now the only information we have required about the problem was the distribution of risks.  However, to compute $\Prob{h(\bm{x})\neq h'(\bm{x})}$ we need to know more about the learning algorithm.  We consider the realisable perceptron where
\begin{align*}
    \Prob{h(\bm{x})\neq h'(\bm{x})} = \frac{\theta_{hh'}}{2}
\end{align*}
where $\theta_{hh'} = \arccos(\bm{w}_h^\tr\bm{w}_{h'})$ is the angle between the weight vectors corresponding to hypotheses $h$ and $h'$.  For any hypothesis, $h$, with risk $r$, the weight vector can be written
\begin{align*}
    \bm{w}_h = \bm{w}^*\,\cos(\pi\,r) + \bm{w}_h^\perp \,\sin(\pi\,r)
\end{align*}
where $\bm{w}^*$ is a unit vector in the direction of the perfect perceptron and $\bm{w}_h^\perp$ is some unit orthogonal to $\bm{w}^*$.  For two hypotheses with risk $r$
\begin{align*}
    \theta_{hh'} = \cos^2(\pi\,r) + \bm{w}_h^{\perp\tr} \bm{w}_{h'}^\perp \,\sin^2(\pi\,r).
\end{align*}
If there is a large number of features $\bm{w}_h^{\perp\tr} \bm{w}_{h'}^\perp\approx0$ for the vast majority of hypothesis pairs so that $\theta_{hh'} \approx \cos^2(\pi\,r)$.  Ignoring other correlations
\begin{align*}
    v_r^k = r(1-r) - \frac{1}{2\,\pi} \arccos\!\left(\cos^2(\pi\,r)\right)
\end{align*}
and
\begin{align*}
    \Delta_r^k = \frac{2\,\pi\,r\,(1-r)}{\arccos\!\left(\cos^2(\pi\,r)\right)} -1.
\end{align*}

In Figure~\ref{fig:fullApprox} we show the expected ERM risk versus $m/p$ (recall $p$ is the number of features in the perceptron) in the limit when $p\rightarrow\infty$. 
For comparison the annealed approximation is also shown in Figure~\ref{fig:fullApprox}.
Finally, we also show the Gardner solution, which is only defined in this limit \citep{Gardner_1988,engel2001statistical}.  As we can see, our approximation is very close to the Gardner solution in the limit when $m/p$ becomes large.  There are discrepancies for smaller values of $m/p$ due to ignoring other fluctuations.  There will be fluctuations because pairs of training examples $\bm{x}^\mu$ and $\bm{x}^\nu$ will typically have small chance correlations.  These are of order $1/p$, but because there are $m-1$ other training examples, the fluctuations will grow.

\begin{figure}[htbp]
    \centering
    \includegraphics[width=10cm]{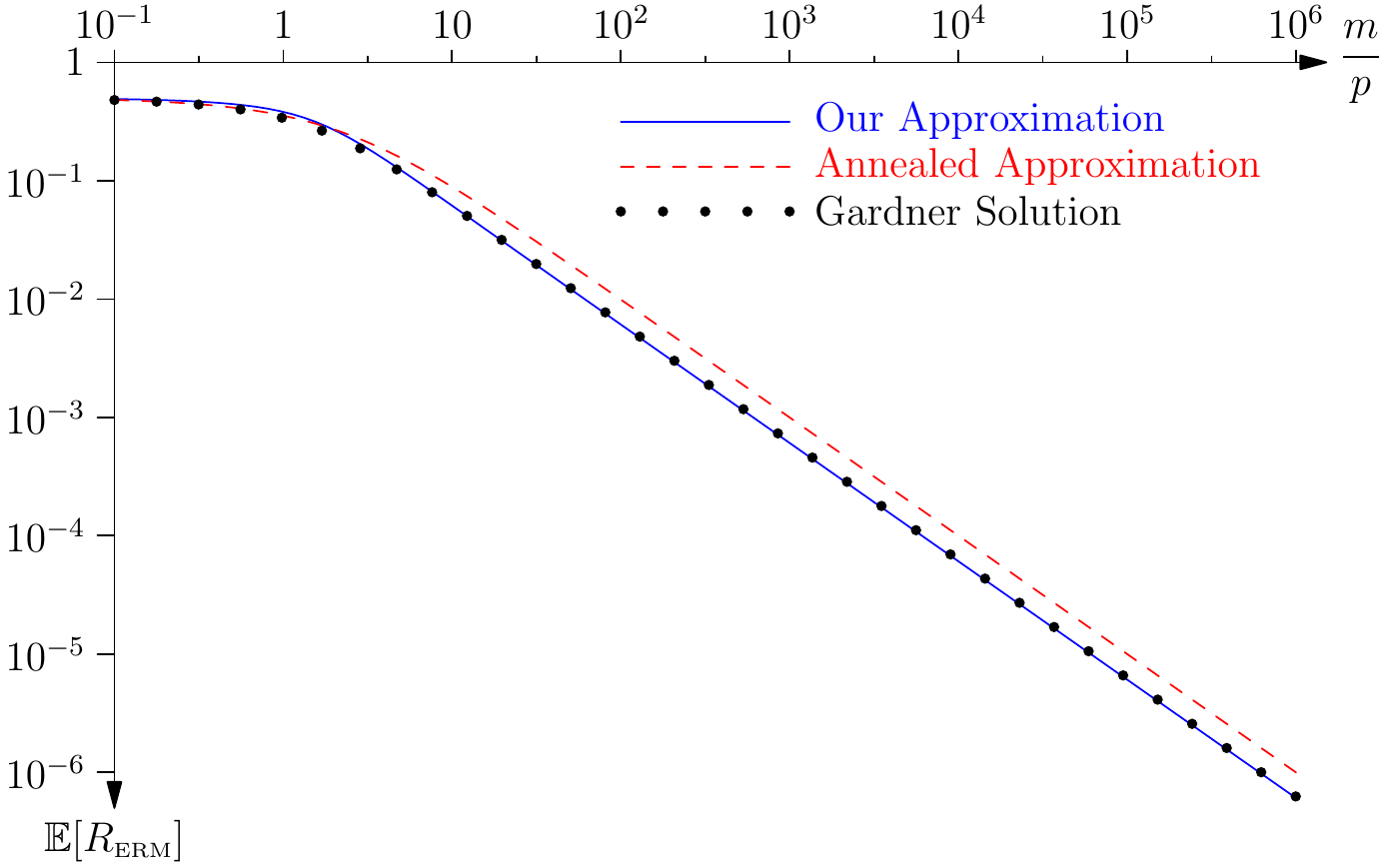}
    \caption{Plot of the expected ERM error versus the ratio $m/p$ for the realisable perceptron in the limit $p\rightarrow\infty$ plotted on a log-log scale.  The solid blue line shows the approximation developed in this paper.  The black dotted line shows the Gardner solution while the red dashed line shows the annealed solution.}
    \label{fig:fullApprox}
\end{figure}

Although the Gardner solution strictly requires us to take the limit $p\rightarrow\infty$ it has been shown that it provides a reasonable approximation to Gibb's learning for perceptrons with a smaller number of features (see Figure~1.4 in \cite{engel2001statistical}).  Gibb's learning for the perceptron can be well approximated by the perceptron learning algorithm with some added noise to ensure that different parts of $H_\erm$ are explored \cite[Section 3.2]{engel2001statistical}.
The Gardner approach has been used to examine other learning rules, noisy training set, etc. (see \cite{engel2001statistical} for a review of the literature).  It has also been extended to SVMs, see, for example, \cite{SVM2001}.  These calculations are very technical and model specific.  In this paper, we have proposed understanding generalisation behaviour more generally through considering $\rho(r)$. However, to obtain results applicable to any learning machine we are forced to use the annealed approximation.

\section{Conclusions}
\label{sec:conclusions}

Traditional machine learning theory has universal applicability in that it provides bounds on the generalisation gap that depend only on the capacity of the learning machine and is independent of the problem being tackled.
This apparent strength is also its weakness.
A learning machine with large capacity may or may not generalise well depending on the distribution of data.
We know there exist distributions of data for which we cannot get any tighter bounds, so to obtain tighter ERM risk bounds requires us to include information about the data distribution.
We have done this by considering the distribution of risks that depends both on the learning machine and on the problem (i.e. the distribution of data).
The cost is that we lose a lot of the elegance of traditional machine learning.
Instead of hard bounds, we are left with approximate results for the expected ERM risk.  There are, however, advantages:  we know that a poorly attuned problem will require a large number of training examples and we have a model for the generalisation performance rather than just the generalisation gap.
We can improve on the annealed approximation, but this requires additional information about the learning machine.
However, the annealed approximation provides a qualitatively accurate model that captures many of the generalisation properties of the exact system.
Most importantly in our view is that generalisation is heavily determined by the attunement. 

The attunement measures the power-law behaviour of $\rho(r)$ around $r=0$.
We show in Appendix~\ref{sec:appAsym} that this will determine the asymptotic learning behaviour (at least for realisable models).
We believe this provides a new language for understanding generalisation performance.
As an illustration of this, consider the case of a perceptron where the input data distribution is confined to a low-dimension subspace of dimension $p'$, which may be much less than the number of features, $p$.
The weights of the perceptron orthogonal to this subspace would not change the output of the perceptron.
As a consequence, the attunement would be $p'-2$---that is, it depends on the dimensionality of the subspace.
In contrast, the capacity of the perceptron is $p+1$---that is, it depends on the number of features.
The capacity is oblivious to the distribution of data!
This provides a simple, but stark example of where attunement can capture features of the learning problem that determine the generalisation performance that capacity fails to do. 
To design a successful learning machine it is not necessary to limit the capacity but to obtain good attunement (i.e. ensure that there is a relatively high proportion of low-risk machines).
We believe this shift in thinking will aid the design of learning machines in the future.

\appendix

\vspace*{1cm}

{\noindent\huge \textbf{Appendices}}

\section{Probably Approximately Correct Bound}
\label{sec:appPAC}

To make a connection to statistical learning theory it is interesting to obtain a PAC bound showing that for sufficiently large $m$ our learning algorithm will return a model with risk no greater than $\epsilon$ with a probability of at least $1-\delta$.
We can compute a bound assuming that annealed approximation is correct.  That is, $p(r,\mathcal{D})=(1-r)^m$.
At least for the realisable perceptron, this is a conservative bound (the drop off is faster than this which will lead to a tighter bound).
We conjecture that the bound for the annealed approximation is generally over pessimistic, but we do not have a proof of this.
Obviously the bound will depend on the distribution of risks.  We consider here the \brisk{} model, i.e. $\rho(r) = \mathrm{Beta}(r|a,a)$.

We prove the following lemma
\begin{lemma}\label{lemma:pac}
  We consider a learning problem where the distribution of risks is beta-distributed
  $\rho(r) = \mathrm{Beta}(r|a,a)$ and where
  $p(r,\mathcal{D})=(1-r)^{m}$.  For any $\epsilon>0$ and
  $\delta>0$, if $m$, satisfy
\begin{align}
    m \geq \frac{2\,(a + \ln(1/\delta))}{\epsilon} + 2\,a -b +1 
\end{align}
then $\Prob{R_\erm<\epsilon} \geq 1-\delta$.
\end{lemma}

\begin{proof}
For the \brisk{} model with an infinite hypothesis space
\begin{align}
  \Prob{ R_\erm \geq \epsilon}
  = \int_{\epsilon}^{1} \frac{r^{a-1}\, (1-r)^{b+m-1}}{B(a,
  b+m)} \dd r.
\end{align}
Although this is an incomplete beta function, unfortunately it is too
complicated to directly prove the PAC bound by substitution.  Instead we
use the standard Chernoff construction to obtain a tail bound for the
beta distribution.  As the exponential function is strictly increasing,
we have $\forall \lambda > 0$
\begin{align}
     \Prob{ R_\erm \geq \epsilon } = \Prob{ \e{\lambda\, R_\erm} \geq \e{\lambda\, \epsilon} }.
\end{align}
It follows from Markov's inequality that
\begin{align}
     \Prob{ R_\erm \geq \epsilon } \leq \frac{\av{\e{\lambda
  \,R_\erm}}}{\e{\lambda\, \epsilon}} = \e{-\left(\lambda
  \,\epsilon-\logg{ \av{\e{\lambda\, R_\erm}}}\right)} = \e{- \psi(\epsilon)}
\end{align}
where
\begin{align} \label{psi}
    \psi(\epsilon) = \lambda\, \epsilon - \logg{\av{\e{\lambda\, R_\erm}}}.
\end{align}
This is true for and $\lambda>0$.

Calculating $\logg{\av{\e{\lambda\, R_\erm}}}$ for a beta distribution
is difficult.  Instead we find an upper bound.  Using $1-r \leq \e{-r}$
and extending the range of the integral (since the exponent is positive
this upper bounds the original integral)
\begin{align}
  \av{\e{\lambda\, R_{\erm}}}
  & = \int_{0}^{1} \frac{\e{\lambda\, r}\, r^{a-1}\,
    (1-r)^{b+m-1}}{B(a,b+m)}\, \dd r
  \\
  & \leq  \int_{0}^{\infty} \frac{r^{a-1}\,
    \e{-r\,(b+m-1-\lambda)}}{B(a, b+m)}\,\dd r
  \\
  &= \frac{\Gamma(a)}
    {B(a, b+m)\,(b+m-1-\lambda)^a}.
\end{align}
Using $B(a, b+m) = \Gamma(a)\,\Gamma(b+m)/\Gamma(a+b+m)$ and
\begin{align}
    \frac{\Gamma(a+b+m)}{\Gamma(b+m)} \leq (a + b + m - 1)^a,
\end{align}
it follows that
\begin{align}
  \av{\e{\lambda R_\erm}}
  &\leq \left( \frac{a+b+m-1}{b+m-1-\lambda} \right)^a.
\end{align}
We can now substitute the above into Equation~(\ref{psi}) to get
\begin{align}
  \psi(\epsilon)
  &\geq \max_{\lambda}{\lambda\, \epsilon - a\, \logg{a+b+m-1} +
    a\, \logg{b+m-1-\lambda}}
  \\
  & = \epsilon\,(b + m - 1) - a - a \,
    \logg{\frac{\epsilon\,(a+b+m-1)}{a}}
 \\ 
  & = \epsilon \,(b+m-1) - a -
    a\,\logg{\frac{\epsilon\,(b+m-1)}{a}} - a\logg{1 +
    \frac{a}{b+m-1}}
  \\
  & \geq \epsilon \,(b+m-1) -a -
    a\,\logg{\frac{\epsilon\,(b+m-1)}{a}} -
    a\,\frac{a}{b+m-1}  \label{ineq:log}
  \\
  & \geq \epsilon \, (b+m-1) -a - a\,
    \logg{\frac{\epsilon\,(b+m-1)}{a}} - a\,\epsilon 
    \label{ineq:assump}
 \\
  & = a \left(\frac{\epsilon \, (b+m-1)}{a} -
    \logg{\frac{\epsilon\,(b+m-1)}{a}} -1 - \epsilon\right),
\end{align}
where inequality (\ref{ineq:log}) comes from $-\logg{1+x} \geq -x$,
while (\ref{ineq:assump}) follows as by the assumption made in theorem
$b+m-1 \geq a/\epsilon$.  We now use the fact that $x-\logg{x}>x/2$ for
all $x>0$, thus (using $x=\epsilon \, (b+m-1)/a$)
\begin{align}
    \psi(\epsilon) &> \frac{\epsilon \, (b+m-1)}{2} - a - a\,\epsilon.
\end{align}

For any $m\geq m^* = 2\,(a+\,\logg{1/\delta}+a\,\epsilon)/\epsilon-b+1$ (or $\epsilon (b+m^*-1)/2 \geq a + a\,\epsilon + \logg{1/\delta}$) we have that
\begin{align}
    \psi(\epsilon) > \logg{1/\delta}.
\end{align}
Thus,
\begin{align*}
    \Prob{R_\erm >\epsilon} \leq \e{-\psi(\epsilon)} < \delta.
\end{align*}

\end{proof}

\section{Asymptotic Generalisation Performance}
\label{sec:appAsym}

Consider the case of a realisable problem with an infinite hypothesis space such that a randomly chosen hypothesis has a risk, $R_h$, distributed according to
\begin{align*}
  \rho(r) = r^{a-1} \sum_{i=0}^\infty c_i \, r^i.
\end{align*}
In this scenario a hypothesis $h\in\mathcal{H}_\erm$ will be
distributed according to
\begin{align*}
  f_{R_\erm}(r) = \frac{(1-r)^{m} \, \rho(r)}{\displaystyle \int_0^1
  (1-r')^{m} \, \rho(r')\,\dd r'}.
\end{align*}
The expected generalisation performance is thus given by
\begin{align*}
  \av{R_\erm} &= \frac{\sum\limits_{i=0}^\infty c_i B(a+1+i,m+1)}
  {\sum\limits_{i=0}^\infty c_i B(a+i,m+1)}
  = \frac{c_0\,B(a+1,m+1) + c_1\,B(a+2,m+1) + \cdots}{c_0\,B(a+1,m+1) + c_1\,B(a+2,m+1)+\cdots} \\
  &= \frac{\frac{B(a+1,m+1)}{B(a,m+1)} + \frac{c_1\,B(a+2,m+1)}{c_0\,B(a,m+1)}+\cdots}{1 + \frac{c_1\,B(a+1,m+1)}{c_0\,B(a,m+1)} + \cdots}
  = \frac{a}{m} + O\!\left(\frac{1}{m^2}\right).
\end{align*}
where we have used
\begin{align*}
\begin{split}
    \frac{B(a+i,m+1)}{B(a,m+1)} 
    & = \frac{\Gamma(a+i)\, \Gamma(a+m+1)}{\Gamma(a)\, \Gamma(a+i+m+1)}\\
    & = \frac{a\,(a+1)\,\cdots\,(a+i-1)}{(a+m+1)\,(a+m+2)\,\cdots\,(a+m+i)}.
\end{split}
\end{align*}
Thus, the generalisation error in the limit of large $m$ depends only on the exponents describing the polynomial growth in the distribution of risk.
This exponent provides a measure for the attunement of our learning machine to the problem being studied.

\vskip 0.2in
\bibliography{References}
\end{document}